\documentclass{article}

\usepackage{PRIMEarxiv}

\usepackage[utf8]{inputenc} 
\usepackage[T1]{fontenc}    
\usepackage{hyperref}       
\usepackage{url}            
\usepackage{booktabs}       
\usepackage{amsfonts}       
\usepackage{nicefrac}       
\usepackage{microtype}      
\usepackage{lipsum}
\usepackage{bm}
\usepackage{amsthm}
\usepackage{makecell}
\usepackage{amsmath}
\newtheorem{lemma}{Lemma}
\newtheorem{remark}{Remark}
\newtheorem{theorem}{Theorem}
\usepackage{fancyhdr}       
\usepackage{graphicx}       
\graphicspath{{media/}}     

\title{Time-Varying Formation Tracking Control of Wheeled Mobile Robots With Region Constraint: A Generalized Udwadia-Kalaba Framework
}

\author{
  Yijie Kang, Yuqing Hao , Qingyun Wang \\
  Department of Dynamics and Control \\
  Beihang University \\
  Beijing\\
 \\
   \And
  Guanrong Chen \\
  Department of Electrical Engineering \\
  City University of Hong Kong \\
  Hong Kong SAR\\
   \\
}

\begin{document}
\maketitle

\begin{abstract}
In this article, the time-varying formation tracking control of wheeled mobile robots with region constraint is investigated from a generalized Udwadia-Kalaba framework. The communication network is modeled as a directed and weighted graph that has a spanning tree with the leader being the root. By reformulating the time-varying formation tracking control objective as an equality constrained equation and transforming the region constraint by a diffeomorphism, the time-varying formation tracking controller with the region constraint is designed under the generalized Udwadia-Kalaba framework. Compared with the existing works on time-varying formation tracking control, the region constraint is taken into account in this paper, which ensures the safety of the robots. Finally, the feasibility of the proposed control strategy is illustrated through some numerical simulations.
\end{abstract}


\section{Introduction}

Over the past three decades, coordination control of wheeled mobile robots has generated considerable attention\cite{arai2002advances}. The coordination control of wheeled mobile robots is mostly classified into synchronization control\cite{li2013distributed,wen2012consensus,zhao2018designing,wen2013consensus}, formation control\cite{liu2023distributed,kang2025formation,xu2025distributed}, formation-containment control\cite{liu2023,zhou2023appointed,xu2023adaptive}, and so on. In particular, the formation control of wheeled mobile robots has raised increasing research interest due to its potential applications in military missions and civil engineering, such as rescue search, environmental monitoring, source allocation\cite{xia2024penalty}, target enclosing\cite{li2024fixed,yu2019cooperative,zheng2025enclosing}, and so on.

Generally, formation control can be classified as time-invariant formation control and time-varying formation control. In recent years, various approaches have been proposed to achieve time-invariant formation control\cite{sun2024distributed,qian2019halo,wang2021modular}. Notably, in the application of multiple targets' surveillance using wheeled mobile robots, the time-varying formation is important to avoid obstacles in complicated environments and realize translational, rotational and scaling formation maneuvers simultaneously. Moreover, the methodologies for the time-invariant formation control cannot be directly applied to solving time-varying formation control problems because the time-varying formation typically brings time-dependent challenges to both the analysis and design of the formation control law. To deal with various time-dependent issues, the works in \cite{wu2024formation,romero2023global,chen2024time} address the time-varying formation stabilization problem. Considering time-varying target enclosing control for example, achieving a desired time-varying formation is only the first step to enclose the time-varying target. The wheeled mobile robots are also required to move along with the dynamic target. In such scenarios, time-varying formation tracking control for wheeled mobile robots is required.

To date, extensive research has been conducted on the time-varying formation tracking control problem based on  single-integrator models\cite{xu2020affine,fang2021distributed}, double-integrator models\cite{dong2014time,dong2016time} and unicycle models\cite{brinon2014cooperative,maghenem2017formation}. For example, in \cite{fang2021distributed}, the closed-loop formation maneuver control problem is investigated with nongeneric or nonconvex nominal configurations over directed graphs in 3-D space based on the single-integrator model. In \cite{dong2016time}, a formation tracking protocol is designed for second-order multi-agent systems with switching interaction topologies and applied to solving the target enclosing problem of a multiquadrotor unmanned aerial vehicle (UAV) system. In \cite{brinon2014cooperative}, a novel cooperative
control law is developed to stabilize a fleet of vehicles for a large class of time-varying formations considering a unicycle model for the dynamics. Note that most of the above works treat the robots as particles and ignore the dynamic characteristics of the robots. 

In view of robots' underactuated characteristics, it is more practical to design time-varying formation controllers based on dynamic models. However, little research has been conducted into the time-varying formation tracking control of wheeled mobile robots based on dynamic models. In \cite{sun2021time}, a time-varying optimization-based approach is proposed to achieve distributed time-varying formation control for an uncertain Euler–Lagrange system. In \cite{yu2022adaptive}, the adaptive practical optimal time-varying formation tracking problems of the disturbed high-order multi-agent systems with a noncooperative leader are considered. However, most related works ignore region constraint.

In practice, region constraint exists widely due to the limitation of real-world environments, task requirements and safety regulations. For example, aircraft cannot fly in the no-fly zones and manipulators need to operate within strict workspace limits. Ignoring region constraint, the robots might collide with the region boundaries, compromising the safety and integrity of the robots. Thus, it is necessary to consider the region constraint during the process of designing the time-varying formation tracking controller.

Motivated by the above discussions, this paper studies the time-varying formation tracking control of wheeled mobile robots with region constraint from a generalized Udwadia-Kalaba framework. The main contributions are summarized as follows.

(1) The time-varying formation tracking control for the wheeled mobile robots is considered. Compared with \cite{wu2024formation,romero2023global,chen2024time}, which only address time-varying formation stabilization problems, the formation tracking control considered here extends the scope of control tasks.

(2) The dynamic model of the wheeled robots is considered. Compared with \cite{fang2021distributed,dong2016time,brinon2014cooperative}, which consider the unicycle models or integrator models, the general model considered here is more practical.

(3) The region constraint of the robots is considered. Compared with the most existing works on the time-varying formation tracking control \cite{maghenem2017formation,sun2021time,yu2022adaptive}, which ignore the region constraint, the time-varying formation tracking control with region constraint considered here ensures the safety of the robots therefore has higher practical values.

The reminder of this paper is structured as follows. In Section \ref{sec:2}, some preliminaries and the model description are presented. Section \ref{sec:3} introduces the generalized Udwadia-Kalaba formulation. In Section \ref{sec:4}, the time-varying formation tracking controller with region constraint is designed using the generalized Udwadia-Kalaba formulation. Simulation examples are demonstrated in Section \ref{sec:5}. Section \ref{sec:6} concludes this paper.

\section{Preliminaries and model description}
\label{sec:2}

\subsection{Notation}
$\mathbb{R}^{n}$ denotes the $n$-dimensional Euclidean space. Let the $n$-dimensional identity matrix be represented by $ {I_n} \in \mathbb{R}^{n \times n}$. Let the superscript ``$+$'' denote the Moore-Penrose generalized inverse of a matrix, and ``$\otimes$'' denote the Kronecker product. For a vector $x=[x_1,x_2,\cdots ,x_n]^T$, $diag(x)=diag\{x_1,x_2,\cdots ,x_n\}$. For a complex number $k$, $\mathcal{R}(k)$ and $\mathcal{I}(k)$ represent its real and imaginary parts, respectively. $N(A)$ denotes the null-space of matrix $A$. Let $span\{w_1,w_2,\cdots,w_n\}$ denote the linear space spanned by the vectors $w_1,w_2,\cdots,w_n$. The rank of matrix $A$ is represented by $rank(A)$. Let $\bm{1}_n \in \mathbb{R}^n$ and $\bm{0}_n \in  \mathbb{R}^n$ represent the $n$-dimensional all-ones and all-zeros column vectors, respectively.

\subsection{Graph Theory}
Consider a group of $n$ wheeled mobile robots whose communication structure is described by a weighted directed graph $\mathcal{G}=(\Omega,\Theta,\mathcal{A})$. Here, $\Omega={1,2,\ldots,n}$ denotes the set of nodes, and $\Theta\subseteq \Omega\times\Omega$ is the set of directed edges. If $(j,i) \in \Theta$, node $i$ can access information transmitted by node $j$. The weighted adjacency matrix is denoted by $\mathcal{A}=\left[ a_{ij} \right] \in \mathbb{R}^{n\times n}$, where $a_{ij}>0$ if $(j,i) \in \Theta$, otherwise $a_{ij}=0$. The Laplacian matrix is denoted by $\mathcal{L}=[l_{ij}] \in \mathbb{R}^{n\times n}$, satisfying $l_{ii}=\sum\limits_{j=1}^{n} {a_{ij}}$ and $l_{ij}=-a_{ij},i \neq  j$.

We assume that the weighted and directed graph $\mathcal{G}$ contains a directed spanning tree, which ensures the existence of a root node that has directed paths to all other nodes.

\begin{lemma}\label{lemma1}
	\cite{ren2008distributed} $Rank(\mathcal{L})=n-1$ holds if and only if the directed graph $\mathcal{G}$ contains a spanning tree.
\end{lemma}

\subsection{The Model of A Wheeled Mobile Robot}
Consider a four-wheeled mobile robot depicted in Figure \ref{fig_1}. The generalized coordinates vector is defined as $q=(x,y,\theta)^T$, where $(x,y)$ represents the position of the robot’s mass center and $\theta$ is the heading angle. The dynamic model of the robot can be expressed as follows\cite{sun2016application}:
\begin{equation}\label{e7}
	M(q,t) \ddot{q} = F(q,\dot{q},t) + F^{c}(q,\dot{q},t),
\end{equation}
where
\begin{equation*}
	\begin{aligned}
		&M(q,t) = \begin{bmatrix}
			md\cos \theta & md\sin \theta & 0 \\
			0 & 0 & \frac{Jd}{l} \\
			\sin \theta & -\cos \theta & 0 \\
		\end{bmatrix}, \\
		&F(q,\dot{q},t) = \begin{bmatrix}
			md\sin \theta \dot{\theta}\dot{x} - md\cos \theta \dot{\theta}\dot{y} \\
			0 \\
			-\cos \theta \dot{\theta}\dot{x} - \sin \theta \dot{\theta}\dot{y} \\
		\end{bmatrix}, \\
		&F^{c}(q,\dot{q},t) = PU = \begin{bmatrix}
			1 & 1 \\
			1 & -1 \\
			0 & 0 \\
		\end{bmatrix} \begin{bmatrix}
			u_r \\
			u_l \\
		\end{bmatrix}.
	\end{aligned}
\end{equation*}
in which $m$ and $J$ are the mass and the rotational inertia of the mobile robot, respectively. The parameter $l$ is the distance from the wheel axis to the robot center, and $d$ is the wheel radius. Moreover, $u_l$ and $u_r$ represent the driving torques of the motor on the rear wheel.

\begin{figure}[t]
\centerline{\includegraphics[width=0.5\textwidth]{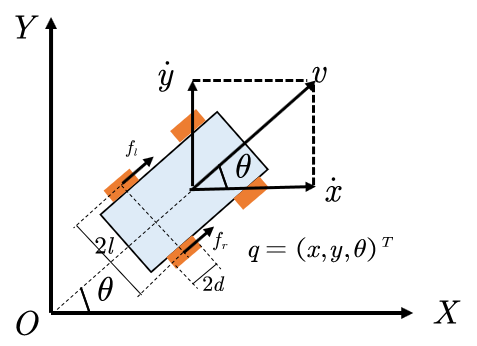}}
	\caption{Model of four-wheeled mobile robot.}
	\label{fig_1}
\end{figure}

\section{Generalized Udwadia-Kalaba formulation}
\label{sec:3}
The Udwadia–Kalaba (U-K) formulation was developed within analytical mechanics for modeling and controlling complex mechanical systems\cite{udwadia2002foundations}. It can be used to obtain analytic solutions of the constraint forces for the mechanical systems subject to holonomic or nonholonomic constraints. To date, the U-K formulation has been successfully employed in a wide range of control applications, such as trajectory tracking\cite{yu2016novel}, synchronization control\cite{wang2021synchronization}, formation control \cite{kang2025formation}, and so on. However, the traditional U-K formulation cannot be used to solve an inequality constraint problem. In \cite{zhang2023novel}, a generalized Udwadia-Kalaba (GUK) formulation is proposed to model and control a mechanical system incorporating both the equality and inequality constraints, which are introduced  below.

\subsection{Equality Constraints}
Consider the following unconstrained mechanical system:
\begin{equation}\label{e2}
	{M}\left( {q,}t \right) {\ddot{q}}={F}\left( {q,\dot{q},}t \right),
\end{equation}
where $q \in \mathbb{R}^N$ is the position, $\dot{q} \in \mathbb{R}^N$ is the velocity, $\ddot{q} \in \mathbb{R}^N$ is the acceleration, ${M}\left( {q,}t \right) \in \mathbb{R}^{N\times N}$ denotes the mass matrix, and the external force vector acting on the system is denoted as ${F}\left( {q,\dot{q},}t \right) \in \mathbb{R}^N$.\\ 
It follows from \eqref{e2} that the acceleration of the unconstrained mechanical system is
\begin{equation}\label{e88}
	a(q,\dot{q},t)=\ddot{q}=M^{-1}(q,t)F(q,\dot{q},t).
\end{equation}

Now, suppose that system \eqref{e2} is subjected to the following equality constraints:
\begin{equation}\label{e3}
	{\Psi}\left( {q,\dot{q},}t \right) =\left[ {\psi}_1, {\psi}_2,\cdots ,{\psi}_{{s}} \right] ^T=\bm{0}.
\end{equation}
It is assumed that the equality constraints are sufficiently smooth and thus differentiable with respect to time $t$ to the required order. Differentiating \eqref{e3} with respect to $t$ yields
\begin{equation}\label{e4}
	{A}\left( {q,\dot{q},}t \right) {\ddot{q}}={b}\left( {q,\dot{q},}t \right) ,
\end{equation}
where ${A}\left( {q,\dot{q},}t \right) \in \mathbb{R}^{s\times N},{b}\left( {q,\dot{q},}t \right) \in \mathbb{R}^s.$ 

Consequently, the constrained dynamics can be described as
\begin{equation}\label{e5}
	{M\ddot{q}}={F}\left( {q,\dot{q},}t \right) +{F}^{{c,e}}\left( {q,\dot{q},}t \right) ,
\end{equation}
where
${F}^{{c,e}}={N}^{-\frac{1}{2}}\left( {AM}^{-1}{N}^{-\frac{1}{2}} \right) ^+\left( {b}-{AM}^{-1}{F} \right)$ is the equality constraint force derived from \eqref{e3}.

\begin{remark}
In this paper, $N$ is chosen to be $M^{-2}$ for computational simplicity.
\end{remark}

To this end, the dynamics of the mechanical system under constraints can be written as
\begin{equation}
	M\ddot{q}=F+M^{\frac{1}{2}}(AM^{-\frac{1}{2}})^+(b-AM^{-1}F).
\end{equation}
Alternatively, the acceleration of the constrained mechanical system is denoted as $\ddot{q}=a+a^{c,e}$, where 
\begin{equation}\label{e87}
	a^{c,e}=M^{-\frac{1}{2}}(AM^{-\frac{1}{2}})^{+}(b-AM^{-1}F).
\end{equation}
is the additional acceleration resulted from \eqref{e3}.

In the GUK formulation, the control objectives are transformed to be equality constraints, thus the equality constraint force becomes the control force. Since $F^{c,e}$ is dictated by the imposed motion constraints, the robots’ initial conditions are required to satisfy \eqref{e3}. In practice, however, this requirement may not hold. To address this issue, the stabilization approach in \cite{baumgarte1972stabilization} is adopted, under which the equality constraint \eqref{e3} can be reformulated as
\begin{equation}\label{e6}
	{\ddot{\Psi}}+\alpha {\dot{\Psi}}+\beta {\Psi}=\bm{0},
\end{equation}
where $\alpha >0,\beta >0$. It is noted that ${\Psi}=\bm{0}$ is the asymptotically stable equilibrium point of system \eqref{e6}.

\subsection{Inequality Constraints}

Suppose that the mechanical system \eqref{e2} is subject to the equality constraints \eqref{e3} and the following inequality constraints:
\begin{equation}\label{e84}
	a_i<g_i(q)<b_i,i=1,2,\dots ,l.
\end{equation}
where $a_i$ and $b_i$ are either finite constants or $\pm \infty$.

Then, the acceleration of the constrained mechanical system  subject to equality constraints and inequality constraints is represented as 
\begin{equation}\label{e60}
	\ddot{q}=a+a^{c,e}+a^{c,i}, 
\end{equation}
where $a^{c,i}$ is the additional acceleration resulted from the inequality constraints \eqref{e84}.

Let $a^{c,i}$ be in the form of 
\begin{equation}\label{e85}
	a^{c,i}=({I}-A^{+}A)r,
\end{equation}
where $r \in \mathbb{R}^N$ is to be designed. Multiplying both sides of $\ddot{q}=a+a^{c,e}+a^{c,i}$ by $A$, it follows that
\begin{equation}
	A\ddot{q}=A(a+a^{c,e}+a^{c,i})=b+Aa^{c,i},
\end{equation}
Noting that $AA^{+}A=0$, the above can be rewritten as
\begin{equation}
	A\ddot{q}=b+Aa^{c,i}=b+A[({I}-A^{+}A)r]=b.
\end{equation}

\begin{remark}
	In order to deal with the equality and inequality constraints separately, $a^{c,i}$ should be in the null space of $A$ with the form of \eqref{e85} while the equality constraints are preempted in the range space of $A$ with the form of $A\ddot{q}=b$. However, the equality and inequality constraints cannot be addressed separately if there is no intersection between the analytic domain of the inequality constraints and $N(A)$, that is, the equality constraints will be violated due to the existence of the inequality constraints.
\end{remark}

Next, let $\hat{\phi}_i(\cdot):(a_i,b_i) \rightarrow \mathbb{R},i=1,2,\cdots ,l$, be a diffeomorphism such that,  with $\xi_i=\hat{\phi}_i(g_i(q))=:\phi_i(q)$,
\begin{equation}
	a_i<g_i(q)<b_i \Leftrightarrow -\infty<\xi_i<\infty
\end{equation}
It is noted that, with a suitable diffeomorphism, \eqref{e84} can be satisfied if and only if $\xi_i$ is bounded.

Let $\xi=[\xi_1,\xi_2,\cdots,\xi_l]^T$,$\phi=[\phi_1,\phi_2,\cdots,\phi_l]^T$. By differentiating it with respect to $t$, one obtains 
\begin{equation}
	\dot{\xi}=\frac{\partial\phi}{\partial q}\dot{q},
\end{equation}
\begin{equation}\label{e82}
	\ddot{\xi}=\frac{d}{dt} \left( \frac{\partial\phi}{\partial q}\dot{q} \right)=\left[\frac{d}{dt}\left(\frac{\partial \phi}{\partial q}\right)\right]\dot{q}+\left(\frac{\partial \phi}{\partial q}\right)\ddot{q}
\end{equation}
Applying \eqref{e60}, \eqref{e82} can be rewritten as
\begin{equation}\label{e81}
	\begin{aligned}
		\ddot{\xi}&=\underbrace{\left[\frac{d}{dt}\left(\frac{\partial \phi}{\partial q}\right)\right]\dot{q}+\left(\frac{\partial \phi}{\partial q}\right)(a+M^{-\frac{1}{2}}(AM^{-\frac{1}{2}})^+(b-Aa))}_{p_1}+
		\underbrace{\left(\frac{\partial \phi}{\partial q}\right)(I-A^{+}A)}_{p_2}r\\
		&=p_1(q,\dot{q},t)+p_2(q,\dot{q},t)r.
	\end{aligned}
\end{equation}

\begin{remark}
	It is obvious from \eqref{e81} that, by designing $r$, the boundedness of $\xi$ is satisfied, which in turn enforces the satisfaction of the inequality constraints.
\end{remark}

Therefore, consider the mechanical system \eqref{e2} subject to equality constraints \eqref{e3} and inequality constraints \eqref{e84}. The constraint force is $F^{c}=F^{c,e}+F^{c,i}$, where
\begin{equation}
	\begin{cases}
		F^{c,e}=M^{\frac{1}{2}}(AM^{-\frac{1}{2}})^+(b-AM^{-1}F), \\
		F^{c,i}=M(I-A^+A)r. & 
	\end{cases}
\end{equation}
Then, the generalized Udwadia-Kalaba formulation is 
\begin{equation}
	M\ddot{q}{=}F{+}\overbrace{\underbrace{M^{\frac{1}{2}}(AM^{-\frac{1}{2}})^+(b{-}AM^{-1}F)}_{equality \ constraint}{+}\underbrace{M(I{-}A^+A)r.}_{inequality\ constraint}}^{constraint\ force}
\end{equation}

\section{Time-varying formation tracking control with region constraint}\label{sec:4}
Consider $n$ wheeled robots numbered from 1 to $n$. The generalized coordinate vector of robot $i$ is defined as $q_i=(x_i,y_i,\theta_i)^T$. The communication topology among $n$ robots is indicated by a weighted and directed graph $\mathcal{G}({\Omega},{\Theta},\mathcal{A})$. The Laplacian matrix corresponding to $\mathcal{G}$ is denoted as $\mathcal{L}$.

Let node 0 denote the leader, whose generalized coordinates are $q_0=[x_0,y_0,\theta_0]^T$. Let ${B}=\left[ b_{10},b_{20},\cdots ,b_{n0} \right] ^T$ represent the communication between the leader and the $n$ robots, where $b_{i0}>0$ if the robot $i$ can obtain information from the leader, otherwise $b_{i0}=0$. Let $\bar{B}{=}diag\left({B}\right) $. 

Further, the communication topology of the whole system is represented by $\bar{\mathcal{G}}=\left( \bar{\Omega} ,\bar{\Theta},\bar{\mathcal{A}} \right) $, where $\{\bar{\Omega} \}=\{0\}\cup {\Omega} =\{0\}\cup \{1,2,\cdots ,n\},\bar{\Theta} \subseteq \bar{\Omega} \times \bar{\Omega} $. It is assumed that $\bar{\mathcal{G}}$ has a directed spanning tree with the leader being the root. The Laplacian matrix corresponding to $\bar{\mathcal{G}}$ is denoted as $\bar{\mathcal{L}}$, and $\lambda_i$ is the $i$-th  eigenvalue of $\bar{\mathcal{L}}$. It is easy to verify that
\begin{equation}\label{e34}
	\bar{\mathcal{L}}=\left[ \begin{matrix}
		0&		\bm{0}_{1\times n}\\
		-{B}&		\mathcal{L}+\bar{B}\\
	\end{matrix} \right].
\end{equation}

Based on \eqref{e7}, the dynamics of the overall constrained system can be expressed as
\begin{equation}\label{e35}
	{M\ddot{{q}}}={{F}}+{{F}}^{c},
\end{equation}
where ${{q}}=\left[ {q}_{0}^{T},{q}_{1}^{T},\cdots ,{q}_{{n}}^{T} \right] ^T,{{M}}=diag\{{M}_0,\cdots,{M}_{{n}}\}, F=[F_0^T,F_1^T, \cdots ,F_n^T]^T, F^{c}{=}[F_0^{cT}, F_1^{cT},\cdots,F_n^{cT}]^T$ is the control force, which is designed to drive the robots to achieve the time-varying formation tracking control with region constraint.

\subsection{Time-Varying Formation Tracking Controller Design}
In this subsection, the time-varying formation tracking control objective is transformed to an equality constrained equation. The corresponding equality constraint force $F^{c,e}$ will be designed below.

The time-varying formation tracking control objective is to drive all robots to form the desired time-varying formation while moving along with the leader. 

Let ${{h(t)}}{=}\left[ {h}_{0}^{T}(t),{h}_{1}^{T}(t),\cdots ,{h}_{n}^{T}(t) \right] ^T {\in} \mathbb{R}^{3(n+1)}$ indicate the desired time-varying formation, where ${h}_0(t)=\bm{0}$. It is assumed that the desired time-varying formation is continuously differentiable.

\begin{remark}
	Compared with \cite{kang2025formation,sun2024distributed,wang2021modular}, which consider the time-invariant formation control, the desired formation considered here is time-varying. The wheeled mobile robots can realize translational, rotational and scaling formation maneuvers simultaneously, which is more practical.
\end{remark}

Here, the control objective is 
\begin{equation}\label{e36}
	\begin{aligned}
		\lim_{t\rightarrow \infty}\left( {q}_i\left( t \right) -{q}_0\left( t \right) \right)=h_i(t)-h_0(t), i \in {\Omega} ,
	\end{aligned}
\end{equation}
Considering the fact that $(q_i(t)-q_0(t))-(q_j(t)-q_0(t))=q_i(t)-q_j(t)=(h_i(t)-h_0(t))-(h_j(t)-h_0(t))=h_i(t)-h_j(t)$, \eqref{e36} can be reformulated as
\begin{equation}
	\lim_{t\rightarrow \infty}a_{ij}\left( {q}_i\left( t \right) -{q}_j\left( t \right) \right)=a_{ij}(h_i(t)-h_j(t)), i,j \in \bar{\Omega}.
\end{equation}
It can be further rewritten as 
\begin{equation}\label{e37}
	\left( \bar{\mathcal{L}}\otimes {I_3} \right) \left( {{q}}-{{h}} \right)  =\bm{0}.
\end{equation}

Applying Baumgarte modification, \eqref{e37} can be further modified as 
\begin{equation}\label{e38}
	\begin{aligned}
		&(\bar{\mathcal{L}}{\otimes}{I_3}){\ddot{{q}}}{=}(\bar{\mathcal{L}}\otimes I_3)\ddot{h}{-}\alpha[(\bar{\mathcal{L}}+I_{n+1}){\otimes}I_3](\bar{\mathcal{L}}{\otimes}I_3)(\dot{q}-\dot{h}){-}
		\beta[(\bar{\mathcal{L}}{+}I_{n+1}){\otimes}I_3](\bar{\mathcal{L}}{\otimes}I_3)({q}{-}h),\\
	\end{aligned}
\end{equation}
where $\alpha>0,\beta>0$.
Following from \eqref{e6}, it obtains that
\begin{equation}\label{e39}
	\begin{aligned}
		&{F}^{c,e}{=}{M} ( \bar{\mathcal{L}} {\otimes}{I_3}) ^{+}\{(\bar{\mathcal{L}}\otimes I_3)\ddot{h}{-}
		\alpha[(\bar{\mathcal{L}}+I_{n+1}){\otimes}I_3](\bar{\mathcal{L}}{\otimes}I_3)(\dot{q}-\dot{h}){-}
		\beta[(\bar{\mathcal{L}}{+}I_{n+1}){\otimes}I_3](\bar{\mathcal{L}}{\otimes}I_3)({q}{-}h){-}(\bar{\mathcal{L}}\otimes I_3)M^{-1}F\}.\\
	\end{aligned}
\end{equation}
It is easy to verify that
\begin{equation}\label{e80}
	\begin{aligned}
		F^{c,e}{=}&M(\bar{\mathcal{L}}\otimes I_3)^{+}(\bar{\mathcal{L}}\otimes I_3)[\ddot{h}{-}\alpha(\dot{q}{-}\dot{h}){-}\beta(q{-}h){-}M^{-1}F]
		{+}M(\bar{\mathcal{L}}\otimes I_3)^{+}(\bar{\mathcal{L}}\otimes I_3)(\bar{\mathcal{L}}\otimes I_3)[{-}\alpha(\dot{q}{-}\dot{h}){-}\beta(q{-}h)].
	\end{aligned}
\end{equation}

Let $\eta_j=\sum\limits_{i=1}^{n+1}l_{ij},j=1,2,\cdots,n+1$, and $\eta=(\eta_1,\eta_2,\cdots ,\eta_{n+1})^T \in \mathbb{R}^{n+1}$. From the properties of $\bar{\mathcal{L}}$, it can be obtained that
\begin{equation}\label{e40}
	\begin{array}{l}
		\bar{\mathcal{L}}^{+}\bar{\mathcal{L}}=I_{n+1}-Q_{n+1},\\
		\bar{\mathcal{L}}^{+}\bar{\mathcal{L}}^2=\bar{\mathcal{L}}-\frac{1}{n+1}(\bm{1}_{n+1}\otimes \eta^T).\\
	\end{array}
\end{equation}
where ${Q}_{n+1}=\frac{1}{n+1}\bm{1}_{n+1}\bm{1}_{n+1}^T$ . Combining \eqref{e80} and \eqref{e40}, it can be obtained that
\begin{equation}\label{e101}
	\begin{aligned}
		F^{c,e}&{=}M[(I_{n+1}{-}Q_{n+1}){\otimes}I_3][\ddot{h}{-}\alpha(\dot{q}{-}\dot{h}){-}\beta(q{-}h){-}
		M^{-1}F]{+}M(\bar{L}{\otimes}I_3)[{-}\alpha(\dot{q}{-}\dot{h}){-}\beta(q{-}h)]{-}\\
		&\frac{M}{n+1}[(\bm{1}_{n+1}\otimes \eta^T)\otimes I_3][{-}\alpha(\dot{q}{-}\dot{h}){-}\beta({q{-}h})].
	\end{aligned}
\end{equation}
Following from \eqref{e7}, the driving torque $U^{c,e}$ can be calculated as
\begin{equation}
	U^{c,e}=(P \otimes I_{n+1})^{+}F^{c,e}.
\end{equation}

\begin{theorem}\label{th1}
	Consider the wheeled mobile robots modeled as \eqref{e35}. With the control protocol \eqref{e101}, where\\ $\frac{\alpha^2}{\beta}{>}{\max\limits_{1\leq i\leq n+1}}\frac{\mathcal{I}^{2}(\lambda_{i})}{(\mathcal{R}(\lambda_{i})+1)\mathcal{I}^{2}(\lambda_{i})+(\mathcal{R}(\lambda_{i})+1)^3}$ and $\alpha>0,\beta>0$, the time-varying formation tracking control of wheeled mobile robots \eqref{e37} can be achieved.
\end{theorem}

\begin{proof}
	By substituting \eqref{e80} into \eqref{e35}, it yields that
	\begin{equation}\label{e19}
		\begin{aligned}
			&{M\ddot{q}}{=}{F}{+}M(\bar{\mathcal{L}}{\otimes}I_3)^{+}(\bar{\mathcal{L}}{\otimes}I_3)[\ddot{h}{-}\alpha(\dot{q}{-}\dot{h}){-}\beta(q{-}h){-}{M}^{-1}F]
			{+}M(\bar{\mathcal{L}}\otimes I_3)^{+}(\bar{\mathcal{L}}\otimes I_3)(\bar{\mathcal{L}}\otimes I_3)[{-}\alpha(\dot{q}{-}\dot{h}){-}\beta(q{-}h)].
		\end{aligned}
	\end{equation}
	By multiplying \eqref{e19} on the left by $(\bar{\mathcal{L}} \otimes {I_3})M^{-1}$, one gets
	\begin{equation}\label{e20}
		\begin{aligned}
			(\bar{\mathcal{L}}{\otimes}{I_3})\ddot{q}&{=}(\bar{\mathcal{L}}\otimes I_3)M^{-1}F{+}(\bar{\mathcal{L}} {\otimes}{I_3})[\ddot{h}{-}\alpha(\dot{q}{-}\dot{h}){-}\beta(q{-}h)
			{-}M^{-1}F]{+}(\bar{\mathcal{L}}\otimes I_3)(\bar{\mathcal{L}}\otimes I_3)[{-}\alpha(\dot{q}{-}\dot{h}){-}\beta(q{-}h)]\\
			&=(\bar{\mathcal{L}}\otimes I_3)\ddot{h}{-}\alpha[(\bar{\mathcal{L}}{+}I_{n+1})\otimes I_3](\bar{\mathcal{L}}\otimes I_3)(\dot{q}{-}\dot{h}){-}
			\beta[(\bar{\mathcal{L}}{+}I_{n+1})\otimes I_3](\bar{\mathcal{L}}\otimes I_3)(q{-}h).
		\end{aligned}
	\end{equation}
	Define the time-varying formation tracking error of the $i$th robot as
	\begin{equation}\label{e21}
		{e}_i\left( t \right){=}\sum_{j=0}^n{a_{i(j+1)}}\left[ \left( {q}_i\left( t \right) {-}{q}_j\left( t \right) \right) {-}\left( {h}_i(t) {-}{h}_j(t) \right) \right] ,i\in {\Omega}.
	\end{equation}
	Then, the formation tracking control error of the whole network is
	\begin{equation}\label{e22}
		{e}=\left[\bm{0}_3^T,{e}_{1}^{T},{e}_{2}^{T},...,{e}_{n}^{T} \right] ^T=\left( \bar{\mathcal{L}}\otimes {I}_3\right) \left( {q}-{h} \right). 
	\end{equation}
	Combining \eqref{e20} with \eqref{e22}, it gives that
	\begin{equation}\label{e23}
		{\ddot{e}}=-\alpha[(\bar{\mathcal{L}}+I_{n+1})\otimes I_3]{\dot{e}}-\beta[(\bar{\mathcal{L}}+I_{n+1})\otimes I_3] {e.}
	\end{equation}
	Then, \eqref{e23} can be indicated as 
	\begin{equation}\label{e24}
		\begin{bmatrix}\dot{e}\\\ddot{e}\end{bmatrix}=\begin{bmatrix}\mathbf{0}_{n+1}&{I}_{n+1}\\-\beta (\bar{\mathcal{L}}+I_{n+1})&-\alpha(\bar{\mathcal{L}}+I_{n+1})   \end{bmatrix}\begin{bmatrix}e\\\dot{e}\end{bmatrix}.
	\end{equation}
	The characteristic equation corresponding to \eqref{e24} is
	\begin{equation}\label{e78}
		\begin{aligned}
			&det\begin{bmatrix}s I_{n+1}&{-}I_{n+1}\\\beta(\bar{\mathcal{L}}{+}I_{n+1})&  sI_{n+1}{+}\alpha(\bar{\mathcal{L}}{+}I_{n+1})\end{bmatrix}{=}\prod\limits_{i=1}^{n+1} s^2{+}\alpha(\lambda_i{+}1)s
			{+}\beta(\lambda_i{+}1)=0,
		\end{aligned}
	\end{equation}
	where $\lambda_1,\lambda_2,\cdots,\lambda_{n+1}$ represent the eigenvalues of $\bar{\mathcal{L}}$.
	
	By calculation, the characteristic roots are
	\begin{equation}\label{e25}
		\begin{gathered}
			s_{i1}{=}\frac{1}{2}\left[-\alpha(\lambda_{i}+1)-\sqrt{\alpha^{2}(\lambda_{i}+1)^{2}-4\beta (\lambda_{i}+1)}\right], \\
			s_{i2}{=}\frac{1}{2}\left[-\alpha(\lambda_i+1)+\sqrt{\alpha^2(\lambda_i+1)^2-4\beta (\lambda_i+1)}\right]. 
		\end{gathered}
	\end{equation}
	
	Since the communication topology $\bar{\mathcal{G}}$ having a directed spanning tree with the leader being the root, it can be verified that $\mathcal{R}(\lambda_i)\geq 0, i=1,2,\cdots,n+1$. It is clear that $\lim\limits_{t \rightarrow \infty}e=\bm{0}$ if and only if $\mathcal{R}(s_{ij})<0,i=1,2,\cdots,n+1;j=1,2$. Let 
	\begin{equation}\label{e26}
		\sqrt{\alpha^{2}(\lambda_{i}+1)^{2}-4\beta (\lambda_{i}+1)}=k+m\bm{i},
	\end{equation}
	where $\bm{i}=\sqrt{-1}$, $k$ and $m$ are real.
	
	It follows from \eqref{e25} that $\mathcal{R}(s_{ij})<0$ if and only if $-\alpha(\mathcal{R}(\lambda_i)+1)<k<\alpha(\mathcal{R}(\lambda_i)+1)$, which implies that $\mathcal{R}(\lambda_i) \geq 0$ and $k^2<\alpha^2(\mathcal{R}(\lambda_i)+1)^2$.
	By squaring both sides of \eqref{e26}, it can be obtained that
	\begin{equation}
		\begin{aligned}
			&k^{2}{-}m^{2}{=}\alpha^{2}(\mathcal{R}^{2}(\lambda_{i}){-}\mathcal{I}^{2}(\lambda_{i})){+}(2\alpha^2{-}4\beta)\mathcal{R}(\lambda_{i}){+}\alpha^2{-}4\beta,\\
		    &km=\alpha^{2}\mathcal{R}(\lambda_{i})\mathcal{I}(\lambda_{i})+(\alpha^2-2\beta)\mathcal{I}(\lambda_{i}).
		\end{aligned}
	\end{equation}
	By simple calculation, one obtains that 
\begin{equation}
	\begin{aligned}
		&k^4 - \bigl[ \alpha^{2}(\mathcal{R}^{2}(\lambda_{i}) - \mathcal{I}^{2}(\lambda_{i})) + (2\alpha^2 - 4\beta)\mathcal{R}(\lambda_{i}) + \alpha^2 - 4\beta \bigr] k^2 - \bigl[ \alpha^{2}\mathcal{R}(\lambda_{i})\mathcal{I}(\lambda_{i}) + (\alpha^2 - 2\beta)\mathcal{I}(\lambda_{i}) \bigr]^2 = 0,
	\end{aligned}
\end{equation}
where
\begin{equation*}
	\begin{aligned}
		k_1^2 &= \frac{1}{2} \Bigl\{ \alpha^{2}(\mathcal{R}^{2}(\lambda_{i}) - \mathcal{I}^{2}(\lambda_{i})) +(2\alpha^2 - 4\beta)\mathcal{R}(\lambda_{i}) + \alpha^2 - 4\beta 
		+ \Bigl[ \bigl( \alpha^{2}(\mathcal{R}^{2}(\lambda_{i}) - \mathcal{I}^{2}(\lambda_{i})) +\\& (2\alpha^2 - 4\beta)\mathcal{R}(\lambda_{i}) + \alpha^2 - 4\beta \bigr)^2 
		\quad + 4\bigl[ \alpha^{2}\mathcal{R}(\lambda_{i})\mathcal{I}(\lambda_{i}) + (\alpha^2 - 2\beta)\mathcal{I}(\lambda_{i}) \bigr]^2 \Bigr]^{\frac{1}{2}} \Bigr\}, \\[6pt]
		k_2^2 &= \frac{1}{2} \Bigl\{ \alpha^{2}(\mathcal{R}^{2}(\lambda_{i}) - \mathcal{I}^{2}(\lambda_{i})) + (2\alpha^2 - 4\beta)\mathcal{R}(\lambda_{i}) + \\&\alpha^2 - 4\beta 
		- \Bigl[ \bigl( \alpha^{2}(\mathcal{R}^{2}(\lambda_{i}) - \mathcal{I}^{2}(\lambda_{i})) + (2\alpha^2 - 4\beta)\mathcal{R}(\lambda_{i}) + \alpha^2 - 4\beta \bigr)^2 \\
		&\quad + 4\bigl[ \alpha^{2}\mathcal{R}(\lambda_{i})\mathcal{I}(\lambda_{i}) + (\alpha^2 - 2\beta)\mathcal{I}(\lambda_{i}) \bigr]^2 \Bigr]^{\frac{1}{2}} \Bigr\}.
	\end{aligned}
\end{equation*}

	It is obvious that $k_2^2<k_1^2$, thus $k^2<\alpha^2(\mathcal{R}(\lambda_i)+1)^2$ is valid if $k_1^2<\alpha^2(\mathcal{R}(\lambda_i)+1)^2$ holds. Thus, it can be calculated that the control parameters $\alpha,\beta$ must satisfy the following condition:
	\begin{equation}\label{e77}
		\frac{\alpha^2}{\beta}>\max_{1\leq i\leq n+1}\frac{\mathcal{I}^{2}(\lambda_{i})}{(\mathcal{R}(\lambda_{i})+1)\mathcal{I}^{2}(\lambda_{i})+(\mathcal{R}(\lambda_{i})+1)^3}.
	\end{equation}
	
	Following the above discussions, the conclusion is reached: $\lim\limits_{t \rightarrow \infty}e=\bm{0}$ if and only if $\bar{\mathcal{G}}$ has a directed spanning tree and \eqref{e77} holds. According to Lemma \ref{lemma1}, it obtains that $N(\mathcal{\bar{L}})=span\{\boldsymbol{1}_{n+1}\}$. Thus, $N(\mathcal{\bar{L}} \otimes I_3)=\{\boldsymbol{1}_{3(n+1)} \}$, which implies that $q_i-h_i=q_j-h_j,i,j=1,2,\cdots,n$. Therefore, the time-varying formation tracking control is achieved.
\end{proof}

\begin{remark}
	Compared with \cite{fang2021distributed,dong2016time,brinon2014cooperative}, where the unicycle models and integrator models are considered, the dynamic model of the wheeled robots considered here takes the underactuated characteristic into consideration, thus the designed controller is more practical. In \cite{sun2021time}, the time-varying formation control of uncertain Euler-Lagrange system is studied, where undirected and connected topology is considered. However, the undirected topology requires reciprocal information exchange, which imposes a symmetry condition that is often inconsistent with practical communication settings. In contrast, the weighted directed communication topology considered here provides a more general framework that better reflects realistic communication constraints, which is more practical.
\end{remark}

\subsection{Region-Constrained Controller Design}

In this subsection, the inequality constraint force $F^{c,i}$ is designed as follows.

For simplicity of discussion, it is assumed that all robots are restricted to move within a rectangular region. The region constraint is represented as 
\begin{equation}\label{e76}
	x_o<x_i<x_f,y_o<y_i<y_f, i=0,1,2,\cdots,n
\end{equation}
where $x_o,x_f$ are the $x-$direction region boundaries, and $y_o,y_f$ are the $y-$direction region boundaries. 

To solve the above inequality-constrained problem, an appropriate diffeomorphism is introduced. Let the diffeomorphism be $\xi=\phi(q)=[\xi_0,\xi_1,\cdots,\xi_n]^T$, where $\xi_i=[\xi_{ix},\xi_{iy}]^T$, with
\begin{equation}\label{e73}
	\begin{aligned}
		& \xi_{ix}=tan\left(\frac{\pi}{x_f-x_o}\left(x_i-\frac{x_o+x_f}{2}\right)\right),\\
		&\xi_{iy}=tan\left(\frac{\pi}{y_f-y_o}\left(y_i-\frac{y_o+y_f}{2}\right)\right).
	\end{aligned}
\end{equation}
It is obvious that $\xi_{ix} \rightarrow +\infty$ as $x_i \rightarrow x_f$, $\xi_{ix} \rightarrow -\infty$ as $x_i \rightarrow x_o$, $\xi_{iy} \rightarrow +\infty$ as $y_i \rightarrow y_f$, and $\xi_{iy} \rightarrow -\infty$ as $y_i \rightarrow y_o$. Thus, $\xi_{ix},\xi_{iy}$ stay bounded inside the rectangle if and only if $x_o<x_i<x_f,y_o<y_i<y_f$.

By simple calculation, the Jacobi matrix $\frac{\partial \phi}{\partial q}$ is represented as $\frac{\partial \phi}{\partial q}=diag\left\{\frac{\partial \xi_0}{\partial q_0}, \frac{\partial \xi_1}{\partial q_1},\cdots,\frac{\partial \xi_n}{\partial q_n}\right\}$, where
\begin{equation}
	\frac{\partial\xi_i}{\partial q_i}{=}
	\begin{bmatrix}
		\frac{\partial \xi_{ix}}{\partial x_i} & \frac{\partial \xi_{ix}}{\partial y_i} & \frac{\partial\xi_{ix}}{\partial \theta_i}\\
		\frac{\partial \xi_{iy}}{\partial x_i} & \frac{\partial \xi_{iy}}{\partial y_i} & \frac{\partial \xi_{iy}}{\partial \theta_i}
	\end{bmatrix}.
\end{equation}

Note that collision avoidance occurs only when the distance between the robot and the region boundary is within a specified distance. For this purpose, let the inner boundary parameters defined as $x_b,x_c,y_b,y_c$, where $x_o<x_b<x_c<x_f,y_o<y_b<y_c<y_f$.

In what follows, a theorem is established to make the robots satisfy the region constraint \eqref{e76}.

\begin{theorem}\label{th2}
	\vspace{2pt} 
	Consider the wheeled mobile robots modeled as \eqref{e35}. With the control protocol $F^{c,i}=M(Q_{n+1}\otimes I_3)r^{*}$, where
	\vspace{5pt} 
	\[
	r^{*} = \begin{cases}
		\bm{0}, & (x_i, y_i) \in [x_b, x_c] \times [y_b, y_c], \ \forall i \in \bar{\Omega}, \\
		-{p_2}^{+} \bigl( p_{1} + \gamma_{1} \dot{\xi} + \gamma_{2} \xi \bigr), & \text{otherwise}
	\end{cases}
	\]
	\vspace{5pt} 
	$Q_{n+1} = \dfrac{1}{n+1} \bm{1}_{n+1} \bm{1}_{n+1}^T$, and $\gamma_1 > 0, \gamma_2 > 0$, the region-constrained control of wheeled mobile robots \eqref{e76} can be achieved.
\end{theorem}

\begin{proof}

	Let the desired value $r^{*}$ be
	\begin{equation}\label{e69}
		r^{*}=\left\{\begin{array}{ll}
			\bm{0}, (x_i,y_i) \in [x_b,x_c]\times[y_b,y_c], \forall i \in \bar{\Omega}, \\
			-{p_2}^+\left(p_{1}+\gamma_{1} \dot{\xi}+\gamma_{2} \xi\right) , otherwise.
		\end{array}\right.
	\end{equation}
	where $\gamma_1>0$ and $\gamma_2>0$.
	
	If all the robots are located within $[x_b,x_c]\times[y_b,y_c]$, then $r{=}r^{*}{=}\bm{0}$. This implies that the inequality-constrained force $F^{c,i}{=}M[(I_{n+1}\otimes I_3){-}A^{+}A]r^{*}{=}\bm{0}$. Thus, the region constraint is no need to consider.
	
	If not all the robots are located within $[x_b,x_c]\times[y_b,y_c]$ with $r=r^{*}=-{p_2}^+\left(p_{1}+\gamma_{1} \dot{\xi}+\gamma_{2} \xi\right)$, then
	\begin{equation}\label{e71}
		\ddot{\xi}=-\gamma_1 \dot{\xi}-\gamma_2 \xi.
	\end{equation}\\
	It follows from \eqref{e71} that $\xi$ will converge to $\bm{0}$, which implies that the positions of all the robots will converge to the center of the constrained region.
	
	Based on the above discussions, it follows from \eqref{e69} that $\xi$ will be bounded. Thus, the region constraint \eqref{e76} can be satisfied all the time.
	
	The inequality-constrained force $F^{c,i}$ is 
	\begin{equation}\label{e74}
		\begin{aligned}
			F^{c,i}&{=}M[(I_{n+1}\otimes I_3){-}A^{+}A]r^{*}
			{=}M[(I_{n+1}\otimes I_3){-}(\bar{\mathcal{L}}\otimes I_3)^{+}(\bar{\mathcal{L}}\otimes I_3)]r^{*}.
		\end{aligned}
	\end{equation}
	Combining \eqref{e40} and \eqref{e74}, it can be derived that
	\begin{equation}\label{e105}
		\begin{aligned}
			F^{c,i}&=M[(I_{n+1}\otimes I_3)-(I_{n+1}-Q_{n+1})\otimes I_3]r^{*}
			=M(Q_{n+1}\otimes I_3)r^{*},
		\end{aligned}
	\end{equation}
	where $Q_{n+1}=\frac{1}{n+1}\bm{1}_{n+1}\bm{1}_{n+1}^T$.
	
	Following from \eqref{e7}, the input driving torque $U^{c,i}$ is calculated as 
	\begin{equation}
		U^{c,i}=(P\otimes I_{n+1})^{+}F^{c,i}.
	\end{equation}
	
\end{proof}

In the above subsections, the time-varying formation tracking controller \eqref{e101} and the region-constrained controller \eqref{e105} are designed respectively. Now a time-varying formation tracking controller with region constraint is proposed as follows.

\begin{theorem}\label{th3}
	Consider the wheeled mobile robots modeled as \eqref{e35}. With the controller designed as $F^{c}=F^{c,e}+F^{c,i}$, where
	\begin{align}\label{e102}
		\begin{cases}
			F^{c,e}{=}M[(I_{n+1}{-}Q_{n+1}){\otimes}I_3][\ddot{h}{-}\alpha(\dot{q}{-}\dot{h}){-}\beta(q{-}h)
			{-}M^{-1}F]{+}M(\bar{L}{\otimes}I_3)[{-}\alpha(\dot{q}{-}\dot{h}){-}\beta(q{-}h)]{-}\\
			\frac{M}{n+1}[(\bm{1}_{n+1}\otimes \eta^T)\otimes I_3][{-}\alpha(\dot{q}{-}\dot{h}){-}\beta({q{-}h})],\\
			F^{c,i}=M(Q_{n+1}\otimes I_3)r^{*},
		\end{cases}
	\end{align}
	the time-varying formation tracking control \eqref{e37} and the region-constrained control \eqref{e76} of the wheeled mobile robots can be achieved.
\end{theorem}

\begin{proof}
	Following the proofs of Theorem \ref{th1} and Theorem \ref{th2}, this theorem can be easily proved. Thus, the proof is omitted.
\end{proof}

\begin{remark}
	In \cite{zhang2023novel}, the generalized Udwadia-Kalaba formulation is first proposed to model and control for single pan/tilt device. The equality constraint imposed on pan joint angle $\theta_1$ is modeled as $\theta_1=\frac{\pi}{12}\sin(\frac{\pi}{12})$ while the inequality constraint imposed on tilt joint angle $\theta_2$ is modeled as $-\frac{\pi}{12}<\theta_2(t)<\frac{\pi}{12}$. Compared with \cite{zhang2023novel}, the present paper advances the GUK framework in several significant respects. In terms of equality constraints, the formulation is extended from trajectory tracking of single pan/tilt device to time-varying formation tracking control of mobile robots with a weighted directed communication topology, enabling cooperative behaviors among multiple wheeled mobile robots. In terms of inequality constraints, rather than considering simple motion-boundary constraints, this work integrates practical region constraints to ensure group-level safety, embedding them into the GUK framework via diffeomorphisms which are different and constructed specifically.
\end{remark}

\section{Numerical simulations}\label{sec:5}

To illustrate the effectiveness of the proposed controllers, some numerical simulations are conducted. The parameter values of the four-wheeled robots are listed in Table \ref{tab:table1}. The communication topology and the initial conditions are shown in Figure \ref{fig_2} and Table \ref{tab:table2}, respectively.

The Laplacian matrix of the network is 
\begin{equation}
	\begin{array}{l}
		\bar{\mathcal{L}}=\left[ \begin{matrix}
			0&		0&		0&		0&		0&	\\
			-0.8&		1.3&		0&		0&		-0.5&	\\
			0&		-0.6&		0.9&		-0.3&		0&	\\
			0&		0&		-0.3&		0.3&		0&	\\
			0&		-0.5&		0&		-0.3&		0.8&	\\
		\end{matrix} \right]\\
	\end{array}
\end{equation}

\begin{table*}[!h]%
	\caption{Parameter values of the four-wheeled robot.\label{tab:table1}}
	\begin{tabular*}{\textwidth}{@{\extracolsep\fill} l c @{}}
		\toprule
		\textbf{Parameter} & \textbf{Value} \\
		\midrule
		Mass of the mobile robot & $m = 1\,\mathrm{kg}$ \\
		Radius of the wheels & $d = 0.05\,\mathrm{m}$ \\
		Distance from the wheel axis to the robot’s center & $l = 0.1\,\mathrm{m}$ \\
		Robot’s central moment of inertia & $I = 1\,\mathrm{kg\cdot m^2}$ \\
		\bottomrule
	\end{tabular*}
\end{table*}

\begin{table*}[!h]%
	\caption{Initial states of the wheeled mobile robots.\label{tab:table2}}
	\begin{tabular*}{\textwidth}{@{\extracolsep\fill} c c c c c c c @{}}
		\toprule
		$i$ &
		\makecell{\textbf{$x_i(0)$}\\\textbf{(m)}} &
		\makecell{\textbf{$\dot{x}_i(0)$}\\\textbf{(m/s)}} &
		\makecell{\textbf{$y_i(0)$}\\\textbf{(m)}} &
		\makecell{\textbf{$\dot{y}_i(0)$}\\\textbf{(m/s)}} &
		\makecell{\textbf{$\theta_i(0)$}\\\textbf{(rad)}} &
		\makecell{\textbf{$\dot{\theta}_i(0)$}\\\textbf{(rad/s)}} \\
		\midrule
		0 & 0 & 0 & 0 & 0 & 0 & 0 \\
		1 & -4 & 2 & 4 & 0 & 0 & 0.2 \\
		2 & -4 & 0 & 2 & 2 & $\pi/2$ & 0.8 \\
		3 & -4 & -3 & -2 & 0 & $\pi$ & 1.8 \\
		4 & -4 & 0 & -4 & -4 & $3\pi/2$ & 2 \\
		\bottomrule
	\end{tabular*}
\end{table*}

\begin{figure}[!t]
	\centering
	\includegraphics[width=0.5\columnwidth]{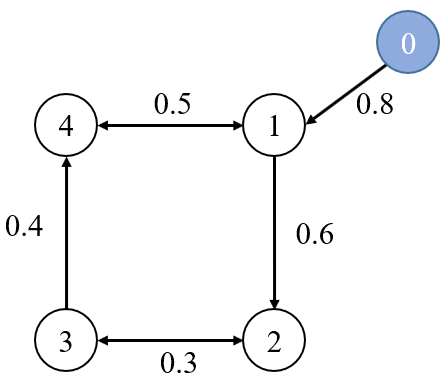}
	\caption{Network topology $\bar{\mathcal{G}}$.}
	\label{fig_2}
\end{figure}

\subsection{Time-Varying Formation Tracking Control Without Region Constraint}
In this subsection, the effectiveness of the proposed time-varying formation tracking controller \eqref{e101} is illustrated.

It is assumed that the trajectory of the leader is
\begin{equation}\label{e115}
	{q }_{0}=\left[ 0.1t,3\sin \frac{2\pi t}{300} ,0 \right] ^T.
\end{equation}

The desired time-varying formation is a circular motion with variable radius, which is described by
\begin{equation}\label{e72}
	\begin{array}{l}
		{h}_0(t)=[0,0,0]^T,\\
		{h}_1(t)=\left[R(t)sin(wt),R(t)cos(wt),wt \right]^T,\\
		{h}_2(t)=\left[R(t)sin(wt+\frac{\pi}{2}) ,R(t)cos(wt+\frac{\pi}{2}),wt+\frac{\pi}{2} \right] ^T,\\
		{h}_3(t)=\left[R(t)sin(wt+\pi),R(t)cos(wt+\pi),wt+\pi \right] ^T,\\
		{h}_4(t)=\left[R(t)sin(wt+\frac{3\pi}{2}),R(t)cos(wt+\frac{3\pi}{2}),wt+\frac{3\pi}{2} \right] ^T,\\
	\end{array}
\end{equation}
where $w=0.6$ is the desired angular velocity.

The simulation time is 470$s$ and the radii of the four robots are given by
\begin{equation}
	R(t)=\left\{\begin{array}{ll}
		4+2\cos(\frac{2\pi t}{500}), t \leq 300, \\
		4+2\cos(\frac{6\pi}{5})+2\sin(\frac{\pi(t-300)}{300}) , 300<t \leq 470.
	\end{array}\right.	
\end{equation}

Let $\alpha=4$, $\beta=0.5$. With the proposed time-varying formation tracking controller \eqref{e101}, The robots’ trajectories are illustrated in Figure \ref{fig_3}. The void dots represent the starting positions of the robots while the trajectories of the robots are represented by black dotted lines. The leader moves along the desired trajectory defined by \eqref{e115}. It is obvious that the robots move along with the leader after the desired time-varying formation is achieved. However, it is noted that some robots exceed the constrained region at the beginning and the end of the motion, which implies that the robots collide with the region boundaries. This, the safety of the robots cannot be guaranteed only with the controller \eqref{e101}. Moreover, the formation tracking errors $||e_i(t)||$ are illustrated in Figure \ref{fig_4}. Clearly, the formation tracking errors converge to 0 quickly. The driving torques $U_{i}^{c,e}$ of the left and right wheels of each robot are shown in Figure \ref{fig_5}. It is obvious that the input torques can converge to 0 fastly when the desired time-varying formation is achieved.

Therefore, with the proposed time-varying formation tracking controller \eqref{e101}, the wheeled mobile robots can achieve the time-varying formation tracking control. However, it is noted that the robots might collide with the region boundaries at the beginning and the end of the motion, therefore, the safety of the robots cannot be guaranteed.

\begin{figure}[!htbp]
	\centering
	\includegraphics[width=0.8\columnwidth]{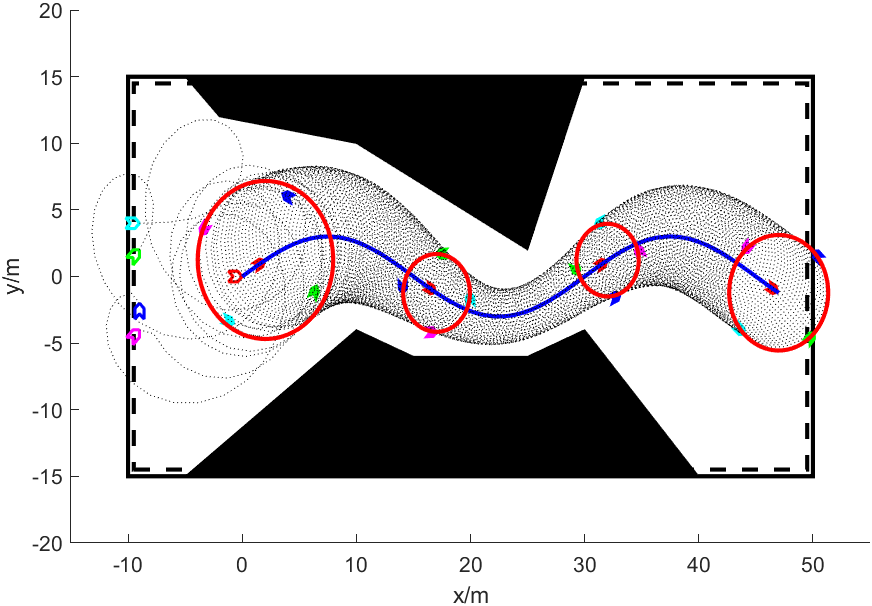}
	\caption{Trajectories of the wheeled mobile robots.}
	\label{fig_3}
\end{figure}

\begin{figure}[!htbp]
	\centering
	\includegraphics[width=0.8\columnwidth]{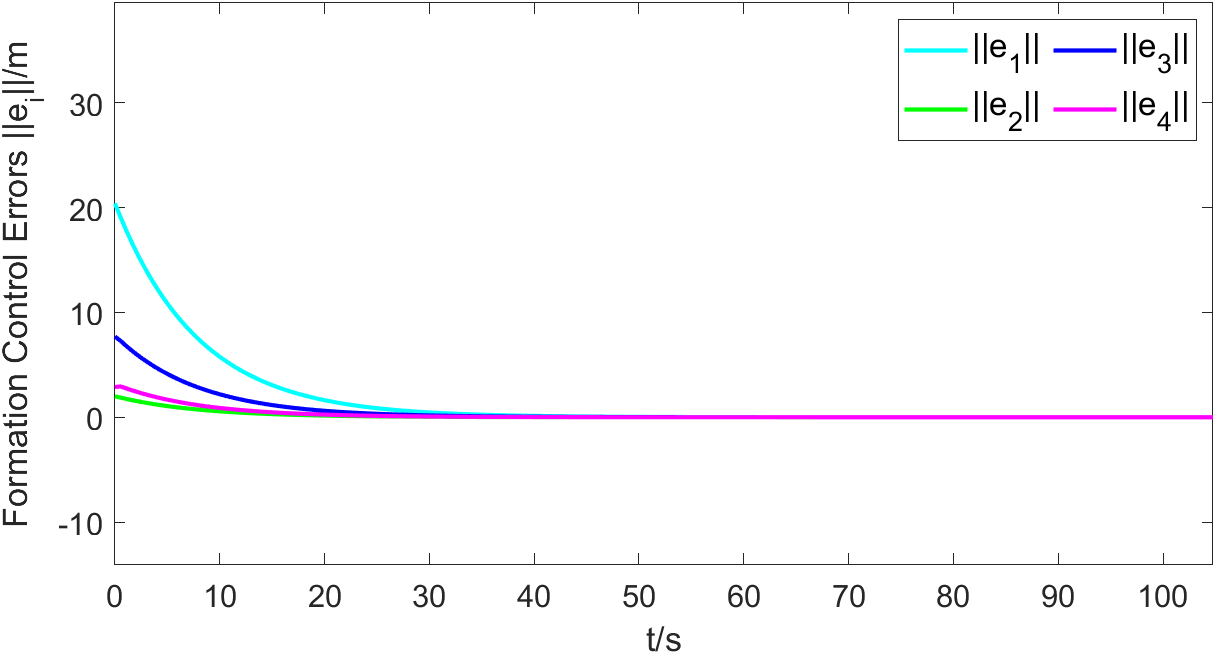}
	\caption{Formation tracking errors of the wheeled robots.}
	\label{fig_4}
\end{figure}

\begin{figure}[!htbp]
	\centering
	\includegraphics[width=0.8\columnwidth]{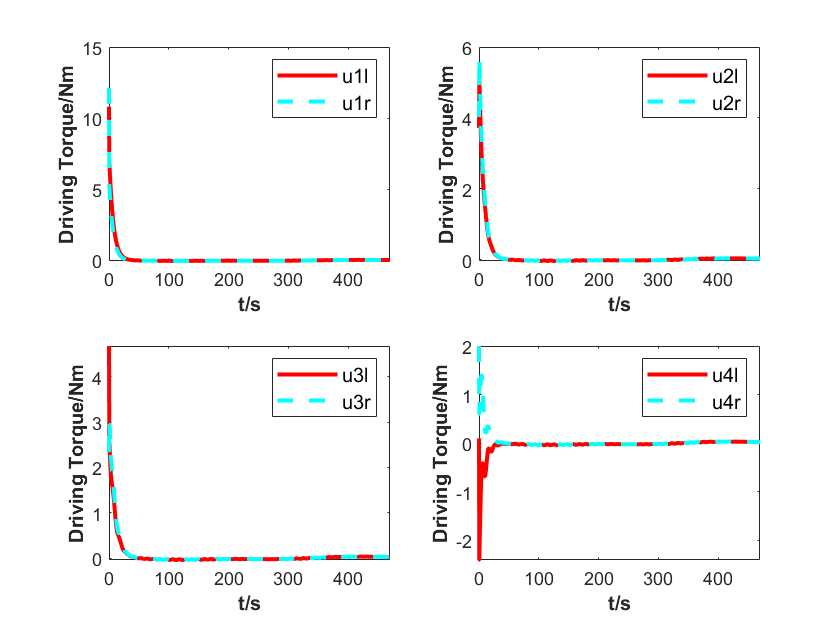}
	\caption{Input driving torques $U_i^{c,e}$.}
	\label{fig_5}
\end{figure}

\subsection{Time-Varying Formation Tracking Control With Region Constraint}
In this subsection, the effectiveness of the region-constrained controller \eqref{e74} is demonstrated.

Considering the size of the robots, the inner boundaries are set $0.5m$ away from the outer boundaries. The parameters in \eqref{e76} are set as $x_o=-10m,x_b=-9.5m,x_c=49.5m,x_f=50m;y_o{=}-15m,y_b{=}-14.5m,y_c{=}14.5m,y_f{=}15m$. In simulation figures, the inner and outer boundaries are represented by the solid line curves and the dashed line curves, respectively.

Let $\alpha=4$, $\beta=0.5$. With the region-constrained formation controller \eqref{e102}, the trajectories of the robots are shown in Figure \ref{fig_6}. Compared with Figure \ref{fig_3}, it is obvious that the trajectories of the robots are restricted within the constrained region, which ensures the safety of the robots. It is noted that the robots are not in the corresponding circular orbits at the end of the motion. If the robots still move in the desired circular orbits, they will collide with the region boundaries. At this time, the equality constraints and the region constraint are in conflict, which cannot be satisfied simultaneously. Thus, the robots violate their desired circular orbits owing to the existence of $F^{c,i}$. 

Figure \ref{fig_7} shows the tracking errors $||e_i(t)||$ of robots. Clearly, the tracking errors converge to 0 rapidly. Moreover, the driving torques $U_{i}^{c,e}$ and $U_{i}^{c,i}$ of the left and right wheels of each robot are shown in Figure \ref{fig_8}. It is obvious that the inequality-constrained torques $U_{i}^{c,i}$ indicated by blue lines, are equal to zero except for a small time interval where the region constraint is not satisfied. 

Therefore, with the region constrained formation controller \eqref{e102}, the time-varying formation tracking control of the wheeled mobile robots with region constraint are achieved.

\begin{figure}[!htbp]
	\centering
	\includegraphics[width=0.8\columnwidth]{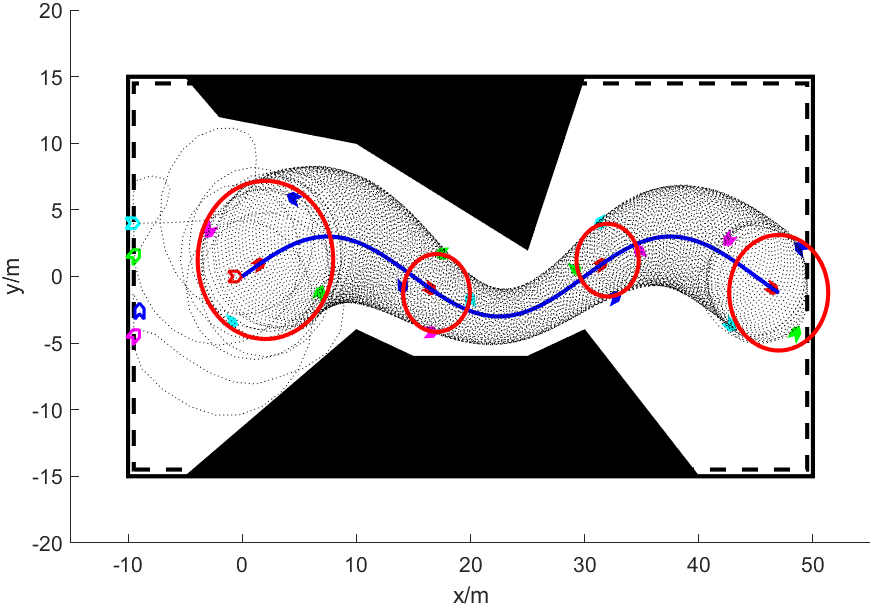}
	\caption{Trajectories of the wheeled mobile robots.}
	\label{fig_6}
\end{figure}

\begin{figure}[!htbp]
	\centering
	\includegraphics[width=0.8\columnwidth]{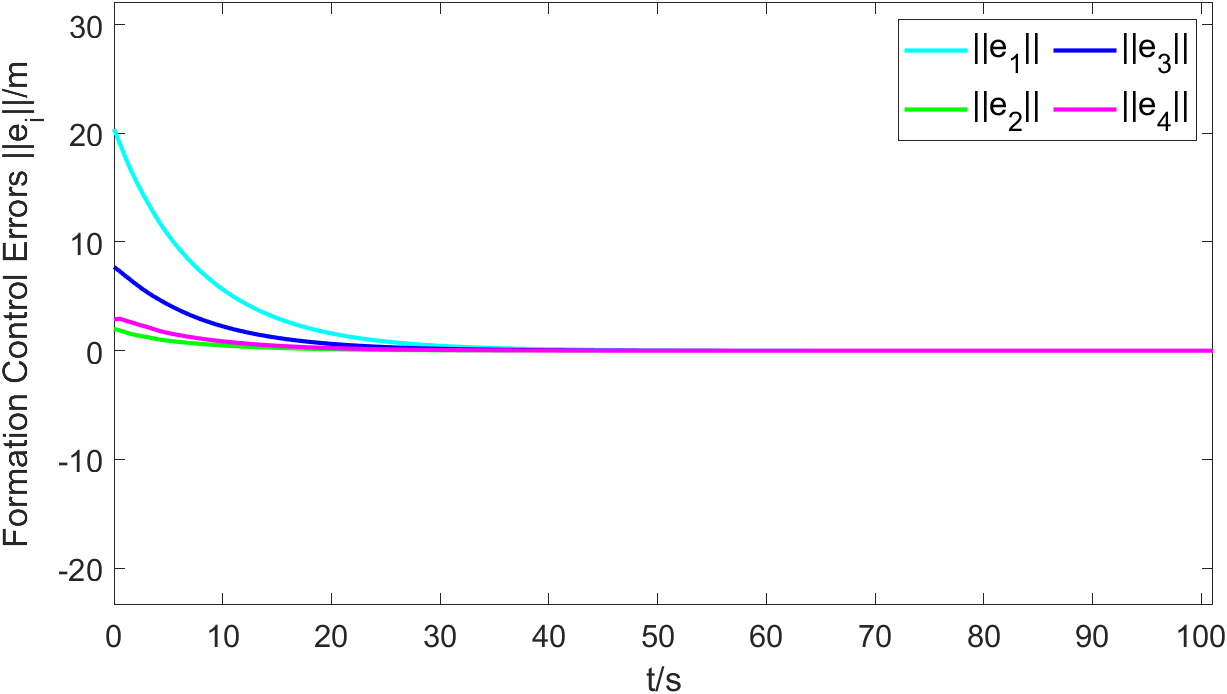}
	\caption{Formation tracking errors of the wheeled robots.}
	\label{fig_7}
\end{figure}

\begin{figure}[!htbp]
	\centering
	\includegraphics[width=0.8\columnwidth]{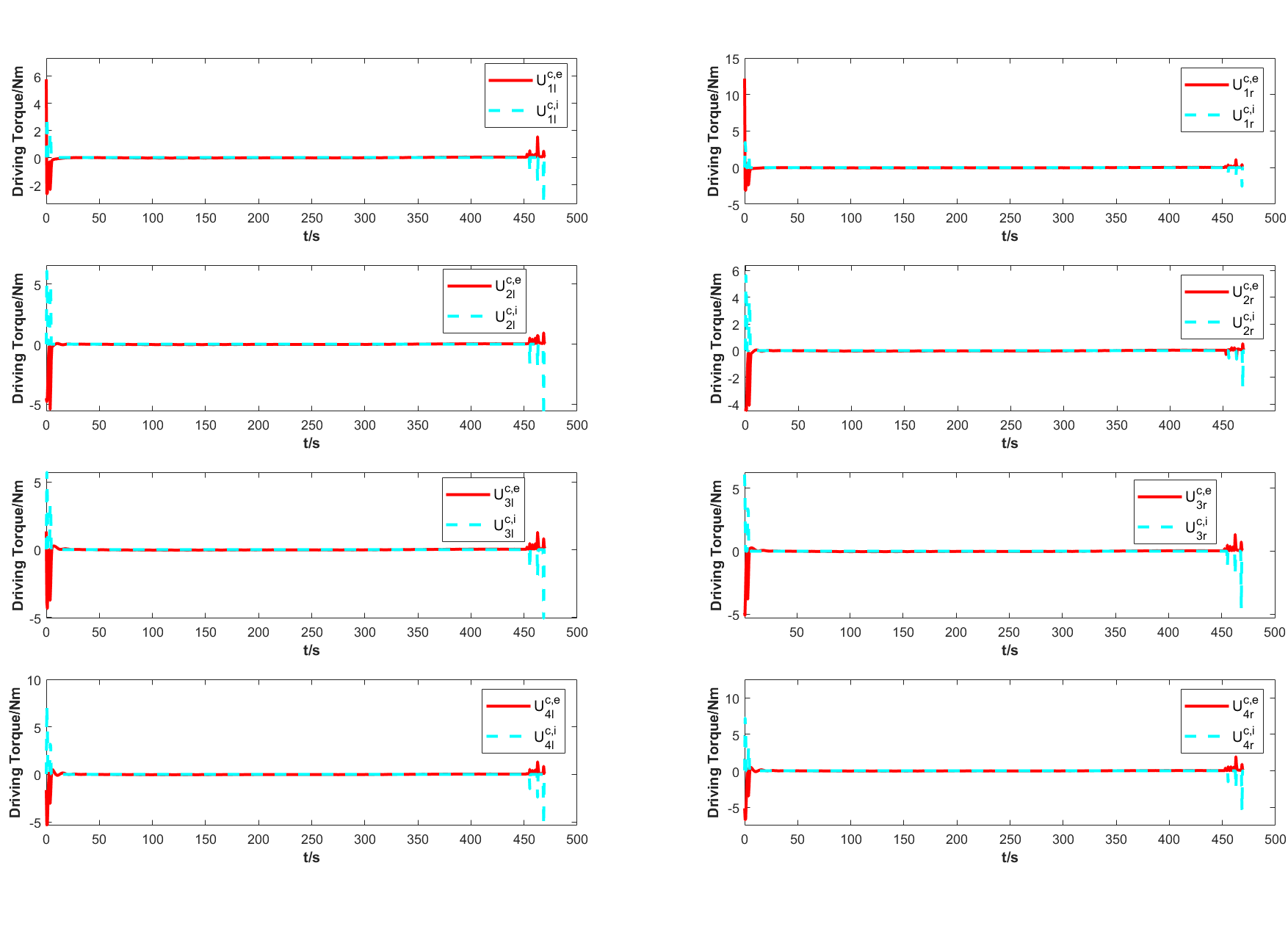}
	\caption{Input driving torques $U_{i}^{c,e}$ and $U_i^{c,i}$.}
	\label{fig_8}
\end{figure}

\begin{figure}[!htbp]
	\centering
	\includegraphics[width=0.8\columnwidth]{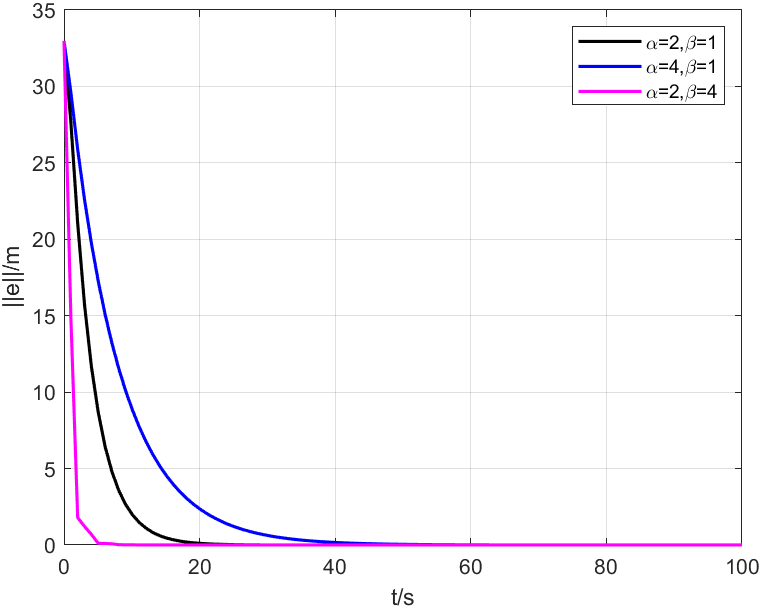}
	\caption{Formation control errors of the robots in various settings with different control parameters.}
	\label{fig_9}
\end{figure}

In what follows, the effects of control parameters $\alpha$ and $\beta$ on the formation tracking control performance are investigated. In order to measure the formation tracking control performance more intuitively, let $||e(t)||=\sum_{i=1}^{4}||e_i(t)||$ denote the total formation tracking control error of the whole system.  Choose three control parameter pairs: (1) $\alpha=2,
\beta=1;$ (2) $\alpha=4, \beta=1;$ (3) $\alpha=2, \beta=4$. The formation tracking control errors in various settings with different control parameter pairs are shown in Figure \ref{fig_9}. It is obvious that the time-varying formation tracking errors can converge to 0 faster with smaller $\alpha$ and larger $\beta$. Following \eqref{e25}, it can be proved that $|Re(\lambda_2)|$ is larger with smaller $\alpha$ and larger $\beta$. Thus, with smaller $\alpha$ and larger $\beta$, the formation tracking control errors converge to 0 faster.

\section{Conclusions}\label{sec:6}
This article investigates the time-varying formation tracking control of wheeled robots with region constraint using the generalized Udwadia-Kalaba formulation. The communication topology is modeled as a weighted directed graph containing a spanning tree. The time-varying formation tracking control objective is transformed to a constraint equation and the region constraint is transformed through an appropriate diffeomorphism. The time-varying formation tracking controller with region constraint is designed under the generalized Udwadia-Kalaba framework. Under the time-varying formation tracking controller, the desired time-varying formation tracking with collision avoidance from the boundaries can be achieved. Finally, the theoretical results are illustrated by some numerical simulations. Future work will investigate the optimization for the time-varying formation tracking control of wheeled mobile robots.

\section*{Acknowledgments}
This work was supported in part by the National Nature Science Foundation of China under Grant 12572004 and 12172020; in part by the Young Elite Scientists Sponsorship Program by CAST under Grant 2022QNRC001; in part by National Key R\&D Program of China: Gravitational Wave Detection Project (No.2024YFC2207900); and in part by the 111 Center under Grant B18002.

\bibliographystyle{unsrt}  
\bibliography{references}

\end{document}